\documentclass[11pt,reqno]{amsart}

\usepackage{fullpage,times,graphicx,amssymb,amsmath,amsfonts,psfrag,xcolor}
\usepackage[colorlinks,citecolor=bbluegray,linkcolor=ddarkbrown,urlcolor=blue,breaklinks]{hyperref}
\usepackage{url}

\author{Vincent Roulet}
\thanks{A shorter, preliminary version of this paper appeared at the NIPS 2015 workshop ``Transfer and Multi-Task Learning: Trends and New Perspectives''.}
\address{INRIA - SIERRA Project Team \& D.I., UMR 8548,\vskip 0ex
\'Ecole Normale Sup\'erieure, Paris, France.}
\email{vincent.roulet@inria.fr}

\author{Fajwel Fogel}
\address{C.M.A.P., \'Ecole Polytechnique, UMR CNRS 7641}
\email{fajwel.fogel@cmap.polytechnique.fr}

\author{Alexandre d'Aspremont}
\address{CNRS \& D.I., UMR 8548, \vskip 0ex
\'Ecole Normale Sup\'erieure, Paris, France.}
\email{aspremon@ens.fr}

\author{Francis Bach} 
\address{INRIA - SIERRA Project Team  \& D.I., UMR 8548,\vskip 0ex
\'Ecole Normale Sup\'erieure, Paris, France.}
\email{francis.bach@inria.fr}

\date{\today}
\subjclass[2010]{}

\usepackage{amsmath,amssymb,hyperref}
\usepackage{epsfig,graphicx}
\usepackage{amsmath,amssymb,hyperref,psfrag}
\usepackage{float}
\usepackage{enumitem}

\usepackage{caption}
\usepackage{tabularx}

\usepackage{todonotes}
\presetkeys{todonotes}{inline, color=yellow!30}{}

\usepackage{algorithmic,algorithm}

\newtheorem{theorem}{Theorem}[section]
\newtheorem{proposition}[theorem]{Proposition}

\renewenvironment{proof}{\textbf{Proof.}}{\QED\bigskip}

\usepackage{algorithmic,algorithm}

\usepackage[square]{natbib}

\definecolor{ddarkbrown}{rgb}{0.5,0.2,0.05} \definecolor{bbluegray}{rgb}{0.05,0,0.5}

\newcommand{\BEAS}{\begin{eqnarray*}}
\newcommand{\EEAS}{\end{eqnarray*}}
\newcommand{\BEA}{\begin{eqnarray}}
\newcommand{\EEA}{\end{eqnarray}}
\newcommand{\BEQ}{\begin{equation}}
\newcommand{\EEQ}{\end{equation}}
\newcommand{\BIT}{\begin{itemize}}
\newcommand{\EIT}{\end{itemize}}
\newcommand{\BNUM}{\begin{enumerate}}
\newcommand{\ENUM}{\end{enumerate}}

\newcommand{\BA}{\begin{array}}
\newcommand{\EA}{\end{array}}

\newcommand{\eg}{{\it e.g.}}
\newcommand{\ie}{{\it i.e.}}

\newcommand{\ones}{\mathbf 1}

\newcommand{\reals}{{\mathbb R}}

\newcommand{\symm}{{\mbox{\bf S}}}  

\newcommand{\Card}{\mathop{\bf Card}}
\newcommand{\Tr}{\mathop{\bf Tr}}
\newcommand{\diag}{\mathop{\bf diag}}

\newcommand{\idm}{\mathbf{I}}

\newcommand{\QED}{~~\rule[-1pt]{6pt}{6pt}}

\newcommand{\argmin}{\mathop{\rm argmin}}

\newcommand{\argmax}{\mathop{\rm argmax}}

\DeclareMathOperator{\x}{x}
\DeclareMathOperator{\y}{y}

\DeclareMathOperator{\EE}{\mathcal{E}}
\DeclareMathOperator{\UU}{\mathcal{U}}
\DeclareMathOperator{\GG}{\mathcal{G}}
\DeclareMathOperator{\PP}{\mathcal{P}}

\DeclareMathOperator{\loss}{\mathbf{loss}}
\DeclareMathOperator{\Loss}{\mathbf{Loss}}
\DeclareMathOperator{\Vect}{\text{Vec}}
\DeclareRobustCommand{\stirling}{\genfrac\{\}{0pt}{}}

\usepackage[off]{auto-pst-pdf} 
\title{Learning with Clustering Structure}

\begin{document}

\maketitle

\begin{abstract}
We study supervised learning problems using clustering constraints to impose structure on either features or samples, seeking to help both prediction and interpretation. The problem of clustering features arises naturally in text classification for instance, to reduce dimensionality by grouping words together and identify synonyms. The sample clustering problem on the other hand, applies to multiclass problems where we are allowed to make multiple predictions and the performance of the best answer is recorded. We derive a unified optimization formulation highlighting the common structure of these problems and produce algorithms whose core iteration complexity amounts to a k-means clustering step, which can be approximated efficiently. We extend these results to combine sparsity and clustering constraints, and develop a new projection algorithm on the set of clustered sparse vectors. We prove convergence of our algorithms on random instances, based on a union of subspaces interpretation of the clustering structure. Finally, we test the robustness of our methods on artificial data sets as well as real data extracted from movie reviews.
\end{abstract}

\section{Introduction}
Adding structural information to supervised learning problems can significantly improve prediction performance. Sparsity for example has been proven to improve statistical and practical performance \citep{Bach2012}. Here, we study clustering constraints that seek to group either features or samples, to both improve prediction and provide additional structural insights on the data. 

When there exists some groups of highly correlated features for instance, reducing dimensionality by assigning uniform weights inside each distinct group of features can be beneficial both in terms of prediction and interpretation \citep{bondell08} by significantly reducing dimension. This often occurs in text classification for example, where it is natural to group together words having the same meaning for a given task \citep{dhillon2003divisive,jiang2011fuzzy}.

On the other hand, learning a unique predictor for all samples can be too restrictive. For recommendation systems for example, users can be partitioned in groups, each having different tastes. Here, we study how to learn a partition of the samples that achieves the best within-group prediction \citep{guzman14,Zhang03}

These problems can of course be tackled by grouping synonyms or clustering samples in an unsupervised preconditioning step. However such partitions might not be optimized or relevant for the prediction task. Prior hypotheses on the partition can also be added as in Latent Dirichlet Allocation \citep{Blei2003} or Mixture of Experts \citep{Jordan94}.  We present here a unified framework that highlights the clustered structure of these problems without adding prior information on these clusters. While constraining the predictors, our framework allows the use of any loss function for the prediction task. We propose several optimization schemes to solve these problems efficiently. 

First, we formulate an explicit convex relaxation which can be solved efficiently using the conditional gradient algorithm \citep{frank1956algorithm,jaggi2013revisiting}, where the core inner step amounts to solving a clustering problem. We then study an approximate projected gradient scheme similar to the Iterative Hard Thresholding (IHT) algorithm \citep{blumensath2009iterative} used in compressed sensing. While constraints are non-convex, projection on the feasible set reduces to a clustering subproblem akin to k-means. In the particular case of feature clustering for regression, the k-means steps are performed in dimension one, and can therefore be solved exactly by dynamic programming \citep{bellman1973note,wang2011ckmeans}. When a sparsity constraint is added to the feature clustering problem for regression, we develop a new dynamic program that gives the exact projection on the set of sparse and clustered vectors.

We provide a theoretical convergence analysis of our projected gradient scheme generalizing the proof made for IHT. Although our structure is similar to sparsity, we show that imposing a clustered structure, while helping interpretability, does not allow us to significantly reduce the number of samples, as in the sparse case for example. 

Finally, we describe experiments on both synthetic and real datasets involving large corpora of text from movie reviews. The use of k-means steps makes our approach fast and scalable while comparing very favorably with standard benchmarks and providing meaningful insights on the data structure.

\section{Learning \& clustering features or samples}\label{sec:pbFormulation}
Given $n$ sample points represented by the matrix $X = (x_1,\ldots,x_n)^T \in \reals^{n\times d}$ and corresponding labels $y = (y_1,\ldots,y_n)$, real or nominal depending on the task (classification or regression), we seek to compute linear predictors represented by $W$. Clustering features or samples is done by constraining $W$ and our problems take the generic form
\BEQ
\BA{ll}
\mbox{minimize} & \Loss(y,X,W) + R(W) \\
\mbox{subject to} & W \in \mathcal{W}, \nonumber
\EA
\EEQ
in the prediction variable $W$, where $\Loss(y,X,W)$ is a learning loss (for simplicity, we consider only squared or logistic losses in what follows), $R(W)$ is a classical regularizer and $\mathcal{W}$ encodes the clustering structure.
 
The clustering constraint partitions features or samples into $Q$ groups $\GG_1,\ldots,\GG_Q$ of size $s_1,\ldots,s_Q$ by imposing that all features or samples within a cluster $\GG_q$ share a common predictor vector or coefficient $v_q$, solving the supervised learning problem. To define it algebraically we use a matrix $Z$ that assigns the features or the samples to the $Q$ groups, \ie~$Z_{iq} = 1$ if feature or sample $i$ is in group $\mathcal{G}_q$ and $0$ otherwise. 
Denoting $V = (v_1,\ldots,v_Q)$, the prediction variable is decomposed as $W = ZV$ leading to the supervised learning problem with clustering constraint
\BEQ\label{eq:gen_form}
\BA{ll}
\mbox{minimize} &  \Loss(y,X,W) + R(W)  \\ 
\mbox{subject to} & W = ZV,\; Z \in \{0,1\}^{m \times Q}, \; Z \ones = \ones,
\EA\EEQ
in variables $W$, $V$ and $Z$ whose dimensions depend on whether features ($m=d$) or samples ($m=n$) are clustered.

Although this formulation is non-convex, we observe that the core non-convexity emerges from a clustering problem on the predictors $W$, which we can deal with using k-means approximations, as detailed in Section~\ref{sec:algo}. We now present in more details two key applications of our formulation: dimensionality reduction by clustering features and learning experts by grouping samples. We only detail regression formulations, extensions for classification are given in the Appendix~\ref{sec:classif}. Our framework also applies to clustered multitask as a regularization hypothesis, and we refer the reader to the Appendix~\ref{sec:multitask} for more details on this formulation.
  
\subsection{Dimensionality reduction: clustering features}\label{ss:cfeat}
Given a prediction task, we want to reduce dimensionality by grouping together features which have a similar influence on the output \citep{bondell08}, e.g. synonyms in a text classification problem. The predictor variable $W$ is here reduced to a single vector, whose coefficients take only a limited number of values. In practice, this amounts to a quantization of the classifier vector, supervised by a learning loss. 

Our objective is to form $Q$ groups of features~$\mathcal{G}_1,\ldots,\mathcal{G}_Q$, assigning a unique weight $v_q$ to all features in group $\mathcal{G}_q$. In other words, we search a predictor $w \in \reals^d$ such that $w_j = v_q$ for all $j \in \mathcal{G}_q$. This problem can be written
\BEQ\label{eq:feat}
\BA{ll}
\mbox{minimize} & \frac{1}{n}\sum_{i=1}^n \loss\left(y_i,w^T x_i \right) + \frac{\lambda}{2} \|w\|_2^2 \\
\mbox{subject to} & w = Zv, \, Z \in \{0,1\}^{d\times Q},\, Z\ones = \ones, 
\EA
\EEQ
in the variables $w \in \reals^{d}$, $v\in \reals^{Q}$ and $Z$. In what follows, $\loss(y_i,w^Tx_i)$ will be a squared or logistic loss that measures the quality of prediction for each sample. Regularization can either be seen as a standard $l_2$ regularization on $w$ with $R(w)= \frac{\lambda}{2}\|w\|^2_2 $, or a weighted regularization on $v$, $R(v) = \frac{\lambda}{2}\sum_{q=1}^Q s_q \|v_q\|_2^2$.

Note that fused lasso \citep{Tibshirani05} in dimension one solves a similar problem that also quantizes the regression vector using an $\ell_1$ penalty on coefficient differences. The crucial difference with our setting is that fused lasso assumes that the variables are ordered and minimizes the total variation of the coefficient vector. Here we do not make any ordering assumption on the regression vector.

\subsection{Learning experts: clustering samples}\label{ss:csamp} 
Mixture of experts \citep{Jordan94} is a standard model for prediction that seeks to learn $Q$ predictors called ``experts", each predicting labels for a different group of samples. For a new sample $x$ the prediction is then given by a weighted sum of the predictions of all experts $\hat y = \sum_{q=1}^Q p_q v_q^Tx$. The weights $p_q$ are given by a prior probability depending on $x$. Here, we study a slightly different setting where we also learn $Q$ experts, but assignments to groups are only extracted from the labels $y$ and not based on the feature variables $x$ as illustrated by the graphical model in Figure~\ref{fig:graphicalModel}.

\begin{figure}[h!]
\begin{center}
\includegraphics{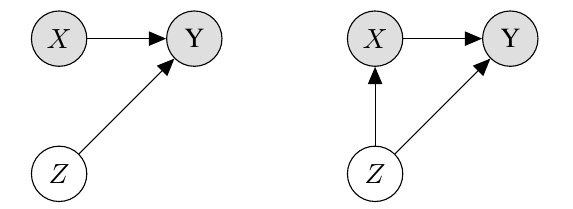}
\caption{Learning multiple diverse experts (\emph{left}), mixture of experts model (\emph{right}). The assignment matrix $Z$ gives the assignment to groups, grey variables are observed, arrows represent dependance of variables.
\label{fig:graphicalModel} }
\end{center}
\end{figure}
This means that while we learn several experts (classifier vectors), the information contained in the features $x$ is not sufficient to select the best experts. Given a new point $x$ we can only give $Q$ diverse answers or an approximate weighted prediction $\hat y = \sum_{q=1}^Q \frac{s_q}{n}v_q^Tx$. Our algorithm will thus return several answers and minimizes the loss of the best of these answers. This setting was already studied by \citet{Zhang03} for general predictors, it is also related to subspace clustering \citep{Elhamifar09}, however here we already know in which dimension the data points lie. 

Given a prediction task, our objective is to find $Q$ groups $\mathcal{G}_1,\ldots,\mathcal{G}_Q$ of sample points to maximize within-group prediction performance. Within each group $\mathcal{G}_q$, samples are predicted using a common linear predictor $v_q$. Our problem can be written
\BEQ\label{eq:cl-smp}
\mbox{minimize}~ \frac{1}{n}\sum_{q=1}^Q\sum_{i \in \mathcal{G}_q} \loss\left(y_i,v_q^T x_i \right) + \frac{\lambda}{2} \sum_{q=1}^Q s_q \|v_q\|^2_2 
\EEQ
in the variables $V = (v_1,\ldots,v_Q) \in \reals^{d \times Q}$ and $\mathcal{G} = (\mathcal{G}_1,\ldots,\mathcal{G}_Q)$ such that $\mathcal{G}$ is a partition of the $n$ samples. As in the problem of clustering features above, $\loss(y_i,v_q^Tx_i)$ measures the quality of prediction for each sample and $R(V) =\frac{\lambda}{2} \sum_{q=1}^Q s_q \|v_q\|^2_2$ is a weighted regularization. Using an assignment matrix $Z \in \{0,1\}^{n\times Q}$ and an auxiliary variable $W = (w_1,\ldots,w_n) \in \reals^{d \times n}$ such that $W = VZ^T$, which means $w_i=v_q$ if $i \in \mathcal{G}_q$, problem~\eqref{eq:cl-smp} can be rewritten
\BEQ
\BA{ll}
\mbox{minimize} & \frac{1}{n}\sum_{i = 1}^n \loss\left(y_i,w_i^T x_i \right) + \frac{\lambda}{2} \sum_{i=1}^n  \|w_i\|^2_2 \\
\mbox{subject to} & W^T = ZV^T, \, Z \in \{0,1\}^{n\times Q}, \, Z\ones = \ones, \nonumber
\EA
\EEQ
in the variables $W \in \reals^{d\times n}$, $V \in \reals^{d\times Q}$ and $Z$. Once again, our problem fits in the general formulation given in~\eqref{eq:gen_form} and in the sections that follows, we describe several algorithms to solve this problem efficiently.

\section{Approximation algorithms}
\label{sec:algo}
We now present optimization strategies to solve learning problems with clustering constraints. We begin by simple greedy procedures and a more refined convex relaxation solved using approximate conditional gradient. We will show that this latter relaxation is exact in the case of feature clustering because the inner one dimensional clustering problem can be solved exactly by dynamic programming.

\subsection{Greedy algorithms}
For both clustering problems discussed above, greedy algorithms can be derived to handle the clustering objective. A straightforward strategy to group features is to first train predictors as in a classical supervised learning problem, and then cluster weights together using k-means. In the same spirit, when clustering sample points, one can alternate minimization on the predictors of each group and assignment of each point to the group where its loss is smallest. These methods are fast but unstable and highly dependent on initialization. However, alternating minimization can be used to refine the solution of the more robust algorithms proposed below.

\subsection{Convex relaxation using conditional gradient algorithm\label{sec:cvxrelax}}
Another approach is to relax the problem by considering the convex hull of the feasible set and use the conditional gradient method (a.k.a.~Frank-Wolfe, \citep{frank1956algorithm,jaggi2013revisiting}) on the relaxed convex problem. Provided that an affine minimization oracle can be computed efficiently, the key benefit of using this method when minimizing a convex objective over a non-convex set is that it automatically solves a convex relaxation, i.e. minimizes the convex objective over the convex hull of the feasible set, without ever requiring this convex hull to be formed explicitly.

In our case, the convex hull of the set $\{W : W = ZV,\, Z \in \{0,1\}^{m \times Q},\, Z\ones = \ones \}$ is the entire space so the relaxed problem loses the initial clustering structure. However in the special case of a squared loss, \ie~$\loss(y,\hat y)=\frac{1}{2}(y-\hat y)^2$, minimization in $V$ can be performed analytically and our problem reduces to a clustering problem for which this strategy is relevant. We illustrate this simplification in the case of clustering features for a regression task, detailed computations and explicit procedures for other settings are given in Appendix~\ref{sec:cvx_relax_other}. 

Replacing $w = Zv$ in \eqref{eq:feat}, the objective function in problem~\eqref{eq:feat} becomes
\BEAS
\phi(v,Z) & = & \frac{1}{2n}\sum_{i=1}^n \left(y_i- (Zv)^T \x_i \right)^2 + \frac{\lambda}{2}\|Zv\|_2^2 \\
& = &\frac{1}{2n}v^TZ^TX^TXZv +\frac{\lambda}{2}v^TZ^TZv - \frac{1}{n}y^TXZv + \frac{1}{2n}y^Ty .
\EEAS
Minimizing in $v$ and using the Sherman-Woodbury-Morrison formula we then get
\BEAS
\min_v \phi(v,Z) & = &\frac{1}{2n}y^T\left(\idm-XZ(Z^TX^TXZ + \lambda n Z^TZ)^{-1}Z^TX^T \right)y \\
 & = & \frac{1}{2n} y^T\left(\idm+\frac{1}{n\lambda}XZ(Z^TZ)^{-1}Z^TX^T \right)^{-1}y,
\EEAS
and the resulting clustering problem is then formulated in terms of the normalized equivalence matrix 
\[
M = Z(Z^TZ)^{-1}Z^T
\] 
such that $M_{ij} = {1}/{s_q}$ if item $i$ and $j$ are in the same group $\mathcal{G}_q$ and $0$ otherwise. 

Writing $\mathcal{M} = \{M : M =Z(Z^TZ)^{-1}Z^T, \: Z \in \{0,1\}^{d \times Q}, \:  Z\ones =\ones\} $ the set of equivalence matrices for partitions into at most $Q$ groups, our partitioning problem can be written
\BEQ\BA{ll}
\mbox{minimize} & \psi(M) \triangleq y^T\left(\idm+\frac{1}{n\lambda}XMX^T \right)^{-1} y \\ 
\mbox{subject to} & M \in \mathcal{M}. \nonumber
\EA\EEQ
in the matrix variable $M\in\symm_n$. 
We now relax this last problem by solving it (implicitly) over the convex hull of the set of equivalence matrices using the conditional gradient method. Its generic form is described in Algorithm~\eqref{algo:GenCondGrad}, where the scalar product is the canonical one on matrices, \ie~ $\langle A,B \rangle = \Tr(A^TB)$. At each iteration, the algorithm requires solving an linear minimization oracle over the feasible set. This gives the direction for the next step and an estimated gap to the optimum which is used as stopping criterion.

\begin{algorithm}[H]
\caption{Conditional gradient algorithm}
\label{algo:GenCondGrad}
\begin{algorithmic}
\STATE Initialize $M_0 \in \mathcal{M}$ 
\FOR{$t = 0,\dots,T$}
	\STATE Solve linear minimization oracle
	\BEQ
	\Delta_t  =  \argmin_{N\in \text{hull}(\mathcal{M})} \left\langle N,\nabla \psi(M_t) \right\rangle \label{eq:oracle}
	\EEQ
	\IF{ gap$(M_t, M_*)  \leq \epsilon $} 
	\RETURN $M_t$ 
	\ELSE 
	\STATE Set $M_{t+1} = M_t + \alpha_t (\Delta_t - M_t)$
	\ENDIF
\ENDFOR
\end{algorithmic}
\end{algorithm} 
The estimated gap is given by the linear oracle as 
\[
\text{gap}(M_t, M_*) \triangleq - \langle \Delta_t-M_t, \nabla \psi(M_t)\rangle.
\]
By definition of the oracle and convexity of the objective function, we have
\[
-\langle \Delta_t-M_t, \nabla \psi(M_t) \rangle \geq -\langle M_*-M_t, \nabla \psi(M_t) \rangle \geq \psi(M_t) - \psi(M_*).
\]
Crucially here, the linear minimization oracle in~\eqref{eq:oracle} is equivalent to a projection step. This projection step is itself equivalent to a k-means clustering problem which can be solved exactly in the feature clustering case and well approximated in the other scenarios detailed in the appendix. For a fixed matrix $M \in \text{hull}(\mathcal{M})$, we have that
\[
P\triangleq -\nabla \psi(M) =\frac{1}{2n^2\lambda} X^T(\idm+\frac{1}{n \lambda}{X}MX^T)^{-1}\y{\y}^T(\idm+\frac{1}{n \lambda}{X}MX^T)^{-1}{X}
\]
is positive semidefinite (this is the case for all the settings considered in this paper). Writing $P^{\frac{1}{2}}$ its matrix square root we get 
\BEAS
\argmin_{N \in \text{hull}(\mathcal{M})} \langle N, \nabla \psi(M) \rangle  & =  & \argmin_{N \in \mathcal{M}} \Tr(N^T\nabla \psi(M)) \\
& = & \argmin_{N \in \mathcal{M}} - \Tr(NP^{\frac{1}{2}}{P^{\frac{1}{2}}}^T) \\
 &  = &\argmin_{N \in \mathcal{M}} \Tr((\idm-N)P^{\frac{1}{2}}{P^{\frac{1}{2}}}^T)) \\
 &  = &\argmin_{N \in \mathcal{M}} \| P^{\frac{1}{2}} - NP^{\frac{1}{2}}\|_F^2\\
 &  = &\argmin_{Z} \min_{V} \|P^{\frac{1}{2}} -ZV \|_F^2,
\EEAS
because $N$ is an orthonormal projection ($N^2 =N$, $N^T=N$) and so is $(I-N)$. Given a matrix $W$, we also have
\BEA\label{eq:kmeans}
\argmin_{Z,V} \|W-ZV\|_F^2 = \argmin \sum_{q=1}^Q \sum_{i\in \mathcal{G}_q} \|w_i - v_q \|_2^2,
\EEA
where the minimum is taken over centroids $v_q$ and partition $(\mathcal{G}_1,\ldots,\mathcal{G}_Q)$.
This means that computing the linear minimization oracle on $\nabla \psi(M)$ is equivalent to solving a k-means clustering problem on $P^{1/2}$. This k-means problem can itself be solved approximately using the k-means++ algorithm which performs alternate minimization on the assignments and the centroids after an appropriate random initialization. Although this is a non-convex subproblem, k-means++ guarantees a constant approximation ratio on its solution \citep{arthur2007}. We write k-means$(V,Q)$ the approximate solution of the projection. Overall, this means that the linear minimization oracle~\eqref{eq:oracle} can therefore be computed approximately. Moreover, in the particular case of grouping features for regression, the k-means subproblem is one-dimensional and can be solved exactly using dynamic programming \citep{bellman1973note,wang2011ckmeans} so that convergence of the algorithm is ensured.

The complete method is described as Algorithm~\ref{algo:condGrad} where we use the classical stepsize for conditional gradient $ \alpha_t = \frac{2}{t+2}$. A feasible solution for the original non-convex problem is computed from the solution of the relaxed problem using Frank-Wolfe rounding, \ie~ output the last linear oracle.

\begin{algorithm}[H]
\caption{Conditional gradient on the equivalence matrix}
\label{algo:condGrad}
\begin{algorithmic}
\REQUIRE $X,y, Q, \epsilon$
\STATE Initialize $M_0 \in \mathcal{M}$ 
\FOR{$t = 0,\ldots,T$}
	\STATE Compute the matrix square root $P^{\frac{1}{2}}$  of $-\nabla \psi(M_0)$
	\STATE Get oracle $\Delta_t = \text{k-means}(P^{\frac{1}{2}},Q)$
	\IF{$-\Tr (\Delta_t-M_t)^T \nabla \psi(M_t) \leq \epsilon$} 
		\RETURN $M_t$ 
	\ELSE 
	\STATE Set $M_{t+1} = M_t + \alpha_t (\Delta_t - M_t)$
	\ENDIF
\ENDFOR
\STATE $Z^*$ is given by the last k-means 
\STATE $V^*$ is given by the analytic solution of the minimization for $Z^*$ fixed
\ENSURE $V^*, Z^*$
\end{algorithmic}
\end{algorithm}

\subsection{Complexity}
The core complexity of Algorithm~\ref{algo:condGrad} is concentrated in the inner k-means subproblem, which standard alternating minimization approximates at cost $O(t_KQp)$, where $t_K$ is the number of alternating steps, $Q$ is the number of clusters, and $p$ is the product of the dimensions of $V$. However, computation of the gradient requires to invert matrices and to compute a matrix square root of the gradient at each iteration, which can slow down computations for large datasets. The choice of the number of clusters can be done given an a priori on the problem (\eg~knowing the number of hidden groups in the sample points), or cross-validation, idem for the other regularization parameters.

\section{Projected Gradient algorithm \label{sec:theory}}
In practice, convergence of the conditional gradient method detailed above can be quite slow and we also study a projected gradient algorithm to tackle the generic problem in~\eqref{eq:gen_form}. Although simple and non-convex in general, this strategy used in the context of sparsity can produce scalable and convergent algorithms in certain scenarios, as we will see below.

\subsection{Projected gradient}
We can exploit the fact that projecting a matrix $W$ on the feasible set 
\[
\{\tilde{W} : \tilde{W} = ZV,\, Z \in \{0,1\}^{m \times Q}, \, Z\ones = \ones\}
\]
is equivalent to a clustering problem, with
\[
\argmin_{Z,V} \|W-ZV\|_F^2 = \argmin \sum_{q=1}^Q \sum_{i\in \mathcal{G}_q} \|w_i - v_q \|_2^2,
\]
where the minimum is taken over centroids $v_q$ and partition $(\mathcal{G}_1,\ldots,\mathcal{G}_Q)$. The k-means problem can be solved approximately with the k-means++ algorithm as mentioned in Section~\ref{sec:cvxrelax}.
We will analyze this algorithm for clustering features for regression in which the projection can be found exactly. Writing k-means$(V,Q)$ the approximate solution of the projection, $\phi$ the objective function and $\alpha_t$ the stepsize, the full method is summarized as Algorithm~\ref{algo:ICLHCL} and its implementation is detailed in Section~\ref{sec:impcomp}.

  \begin{algorithm}[H]
    \caption{Proj. Gradient Descent \label{algo:ICLHCL}}
    \begin{algorithmic}
	\REQUIRE $X,y, Q, \epsilon$
	\STATE Initialize $W_0 = 0$ 
	\WHILE{$|\phi(W_t) - \phi(W_{t-1})| \geq \epsilon $} 
	\STATE $W_{t+\frac{1}{2}} = W_t - \alpha_t(\nabla \Loss(y,X,W_t) + \nabla R(W_t))$
	\STATE $[Z_{t+1},V_{t+1}] = \text{k-means}(W_{t+\frac{1}{2}},Q)$
	\STATE $W_{t+1} = Z_{t+1}V_{t+1}$
	\ENDWHILE
	\STATE $Z^*$ and $V^*$ are given through k-means
	\ENSURE $W^*,Z^*,V^*$	
	\end{algorithmic}
  \end{algorithm}
  
\subsection{Convergence}
We now analyze the convergence of the projected gradient algorithm with a constant stepsize $\alpha_t =1$, applied to the feature clustering problem for regression. We focus on a problem with squared loss without regularization term, which reads
\BEQ\label{eq:least}
\BA{ll}
\mbox{minimize} & \frac{1}{2n}\|Xw-y\|^2_2 \\ 
\mbox{subject to} &  w = Zv, \, Z \in \{0,1\}^{d \times Q},\, Z\ones = \ones \nonumber
\EA\EEQ
in the variables $w \in \reals^d$, $v\in \reals^Q$ and $Z$. We assume that the regression values $y$ are generated by a linear model whose coefficients $w^*$ satisfy the constraints above, up to additive noise, with
\[ 
y = Xw^* + \eta 
\]
where $\eta \sim \mathcal{N}(0,\sigma^2)$. Hence we study convergence of our algorithm to $w^*$, \ie~ to the partition $\GG^*$ of its coefficients and its $Q$ values. 

We will exploit the fact that each partition $\GG$ defines a subspace of vectors $w$, so the feasible set can be written as a union of subspaces. Let $\GG$ be a partition and define
\[
\UU_{\GG} = \{w : w = Zv,\:  Z \in \mathcal{Z}(\GG) \},
\]
where $\mathcal{Z}(\GG)$ is the set of assignment matrices corresponding to $\GG$. Since permuting the columns of $Z$ together with the coefficients of $v$ has no impact on $w$, the matrices in $\mathcal{Z}(\GG)$ are identical up to a permutation of their columns. So, for $Z \in \mathcal{Z}(\GG)$, $\mathcal{Z}(\GG) =\{ Z\Pi, \Pi \: \text{permutation matrix}\}$, therefore $\UU_{\GG}$ is a subspace and the corresponding assignment matrices are its different basis. 

To a feasible vector $w$, we associate the partition $\GG$ of its values that has the least number of groups. This partition and its corresponding subspace are uniquely defined and, denoting $\PP$ the set of partitions in at most~$Q$ clusters, our problem \eqref{eq:least} can thus be written 
\BEQ\label{eq:union}
\BA{ll}
\mbox{minimize} & \frac{1}{2n}\|Xw-y\|^2_2 \\ 
\mbox{subject to} &  w \in \bigcup_{\GG\in \PP} \UU_{\GG}. \nonumber
\EA\EEQ
where the variable $w \in \reals^d$ belongs to a union of subspaces $\UU_{\GG}$.

We will write the projected gradient algorithm for \eqref{eq:union} as a fixed point algorithm whose contraction factor depends on the singular values of the design matrix $X$ on collections of subspaces generated by the partitions~$\GG$. We only need to consider largest subspaces in terms of inclusion order, which are the ones generated by the partitions into exactly $Q$ groups. Denoting $\PP_Q$ this set of partitions, the collections of subspaces are defined as
\BEAS
\EE_1 & = & \{ \UU_{\GG},\, \GG \in \PP_Q \}, \\
\EE_2 & = & \{ \UU_{\GG_1} + \UU_{\GG_2},\, (\GG_1,\GG_2) \in \PP_Q \}, \\
\EE_3 & = & \{ \UU_{\GG_1} + \UU_{\GG_2} + \UU_{\GG_3},\, (\GG_1,\GG_2,\GG_3) \in \PP_Q\}.
\EEAS
Our main convergence result follows. Provided that the contraction factor is sufficient, it states the convergence of the projected gradient scheme to the original vector up to a constant error of the order of the noise.
\begin{proposition}\label{prop:conv}
Given that projection on $\bigcup_{\GG\in \PP} \UU_{\GG}$ is well defined, the projected gradient algorithm applied to \eqref{algo:ICLHCL} converges to the original $w^*$ as
\[ \|w^*-w_t\|_2 \leq \rho^t \|w^*\|_2 + \frac{1-\rho^t}{1-\rho} \nu  \|\eta \|_2 \label{eq:conv},\]
where 
\BEAS
\rho & \triangleq & 2 \max_{\UU \in \EE_3} \| I-\frac{1}{n}\Pi_{\UU}^TX^TX\Pi_{\UU} \|_2 \\
\nu & \triangleq & \frac{2}{n}\max_{\UU \in \EE_2} \|X\Pi_{\UU}\|_2
\EEAS
and $\Pi_{\UU}$ is any orthonormal basis of the subspace $\mathcal{U}$.
\end{proposition}
\begin{proof}
To describe the algorithm we define $\GG_t$ and $\GG_*$ as the partitions associated respectively with $w_t$ and $w^*$ containing the least number of groups and
\BEQ
\def\arraystretch{1.4}
\left\{
\BA{lll}
w_{t+ 1/2} &= & w_t - \nabla \Loss(X,y,w_t) = w_t -\frac{1}{n}X^TX(w_t-w^*) + \frac{1}{n}X^T\eta \\ 
w_{t+1} &= & \argmin_{w \in \bigcup_{\GG\in \PP} \UU_{\GG}} \|w-w_{t+1/2}\|_ 2^2 \\
\UU_t &= & \UU_{\GG_t} \\
\UU_{t,*} & = & \UU_{\GG_t} + \UU_{\GG_*} \\
\UU_{t,t+1,*} & = & \UU_{\GG_t} + \UU_{\GG_{t+1}} + \UU_{\GG_*}.
\EA\nonumber
\right.
\EEQ
Orthonormal projections on $\UU_t$, $\UU_{t,*}$ and $\UU_{t,t+1,*}$ are given respectively by $P_t,P_{t,*},P_{t,t+1,*}$. Therefore by definition $w_t \in \UU_t$, $(w_t,w^*) \in \UU_{t,*}$ and $(w_t,w_{t+1},w^*) \in \UU_{t,t+1,*}$.

We can now control convergence, with
\BEQ\label{eq:conv_eq}
\BA{lll}
\|w^*-w_{t+1}\|_2 & = & \|P_{t+1,*}(w^*-w_{t+1})\|_2 \\
 & \leq & \|P_{t+1,*}(w^*-w_{t+1/2})\|_2 + \|P_{t+1,*}(w_{t+1/2} - w_{t+1})\|_2.
\EA  
\EEQ
In the second term, as $w^* \in \bigcup_{\GG\in \PP} \UU_{\GG} $ and $w_{t+1} = \argmin_{w \in \bigcup_{\GG\in \PP} \UU_{\GG}}  \|w-w_{t+1/2}\|_2$,  we have 
\[
\|w_{t+1}-w_{t+1/2}\|_2^2 \leq \|w^*-w_{t+1/2}\|_2^2
\]
which is equivalent to
\[
\|P_{t+1,*}(w_{t+1}-w_{t+1/2})\|_2^2 + \|(I-P_{t+1,*})w_{t+1/2}\|_2^2 \leq \|P_{t+1,*}(w^*-w_{t+1/2})\|_2^2 +\|(I-P_{t+1,*})w_{t+1/2}\|_2^2
\]
and this last statement implies
\[
\|P_{t+1,*}(w_{t+1}-w_{t+1/2})\|_2 \leq \|P_{t+1,*}(w^*-w_{t+1/2})\|_2.
\]
This means that we get from \eqref{eq:conv_eq}
\BEAS
\|w^*-w_{t+1}\|_2 & \leq & 2\|P_{t+1,*}(w^*-w_{t+1/2})\|_2\\
& = & 2\|P_{t+1,*}(w^* - w_t - \frac{1}{n}X^TX(w^*-w_t) - \frac{1}{n}X^T\eta) \|_2 \\
& \leq & 2\|P_{t+1,*}(I- \frac{1}{n}X^TX)(w^*-w_t)\|_2 + \frac{2}{n}\| P_{t+1,*}(X^T\eta) \|_2  \\
& = & 2\|P_{t+1,*}(I- \frac{1}{n}X^TX)P_{t,*}(w^*-w_t)\|_2 + \frac{2}{n}\| P_{t+1,*}(X^T\eta) \|_2 \\
& \leq & 2\|P_{t+1,*}(I- \frac{1}{n}X^TX)P_{t,*}\|_2 \|w^*-w_t\|_2 + \frac{2}{n}\| P_{t+1,*}X^T\|_2 \|\eta \|_2.
\EEAS
Now, assuming
\BEA
2\|P_{t+1,*}(I-\frac{1}{n}X^TX)P_{t,*}\|_2 & \leq & \rho \label{eq:rho}\\
\frac{2}{n}\|P_{t+1,*}X^T\|_2 & \leq & \nu \label{eq:nu}
\EEA
and summing the latter inequality over $t$, using that $w_0 = 0$, we get
\[ 
\|w^*-w_t\|_2 \leq \rho^t \|w^*\|_2 + \frac{1-\rho^t}{1-\rho} \nu  \|\eta \|_2 .
\]
We bound $\rho$ and $\nu$ using the information of $X$ on all possible subspaces of $\EE_2$ or $\EE_3$. For a subspace $\UU \in \EE_2$ or $\EE_3$, we define $P_{\UU}$ the orthonormal projection on it and $\Pi_{\UU}$ any orthonormal basis of it. For \eqref{eq:nu} we get
\BEAS
\|P_{t+1,*}X^T\|_2 = \|XP_{t+1,*}\|_2 & \leq & \max_{\UU \in \EE_2} \|XP_{\UU}\|_2 = \max_{\UU \in \EE_2} \|X\Pi_{\UU}\|_2,
\EEAS
which is independent of the choice of $\Pi_{\UU}$.

For \eqref{eq:rho}, using that $\UU_{t,*} \subset \UU_{t,t+1,*}$ and $\UU_{t+1,*} \subset \UU_{t,t+1,*}$, we have
\BEAS
\|P_{t+1,*}(I- X^TX)P_{t,*}\|_2 & \leq & \| P_{t,t+1,*}(I-\frac{1}{n}X^TX) P_{t,t+1,*}\|_2 \\
& \leq & \max_{\UU \in \EE_3} \| P_{\UU}(I-\frac{1}{n}X^TX) P_{\UU}\|_2  \\
& = & \max_{\UU \in \EE_3} \| \Pi_{\UU}(I-\frac{1}{n}\Pi_{\UU}^TX^TX\Pi_{\UU}) \Pi_{\UU}^T\|_2 \\
& = & \max_{\UU \in \EE_3} \| I-\frac{1}{n}\Pi_{\UU}^TX^TX\Pi_{\UU}\|_2,
\EEAS 
which is independent of the choice of $\Pi_{\UU}$ and yields the desired result.
\end{proof}
We now show  that $\rho$ and $\nu$ derive from bounds on the singular values of $X$ on the collections $\EE_2$ and $\EE_3$. Denoting $s_{min}(A)$ and $s_{max}(A)$ respectively the smallest and largest singular values of a matrix $A$, we have 
\[
\max_{\UU \in \EE_2} \|X\Pi_{\UU}\|_2 = \max_{\UU \in \EE_2} s_{max}(X\Pi_{\UU}),
\]
and assuming $\UU \in \EE_3$ and that
\[
1-\delta \leq  s_{min}\left(\frac{X\Pi_{\UU}}{\sqrt{n}}\right) \leq s_{max}\left(\frac{X\Pi_{\UU}}{\sqrt{n}}\right) \leq 1 + \delta,
\]
for some $\delta >0$, then 
\citep[Lemma 5.38]{Vers10} shows
\[
\| I-\frac{1}{n}\Pi_{\UU}^TX^TX\Pi_{\UU} \|_2 \leq 3\max\{\delta,\delta^2\}.
\]

We now show that for isotropic independent sub-Gaussian data $x_i$, these singular values depend on the number of subspaces of $\EE_1$, $N$, their dimension $D$ and the number of samples $n$. This proposition reformulates results of \citet{Vers10} to exploit the union of subspace structure.

\begin{proposition}
Let $\EE_1, \EE_2, \EE_3$ be the finite collections of subspaces defined above, let $D = \max_{\UU \in \EE_1}\dim(U)$ and $N = \Card(\EE_1)$. Assuming that the rows $x_i$ of the design matrix are $n$ isotropic independent sub-gaussian, we have 
\[
\frac{1}{\sqrt{n}}\max_{\UU \in \EE_2} \|X\Pi_{\UU}\|_2 \leq 1+\delta_2 +\epsilon 
\quad\mbox{and}\quad
\max_{\UU \in \EE_3} \| I-\frac{1}{n}\Pi_{\UU}^TX^TX \Pi_{\UU} \|_2  \leq  3\max\{\delta_3+\epsilon,(\delta_3+\epsilon)^2\},
\]
with probability larger than $1- \exp(-c\epsilon^2 n)$, where $\delta_p = C_0\sqrt{\frac{pD}{n}} + \sqrt{\frac{1+p\log(N)}{cn}}$, $\Pi_{\UU}$ is any orthonormal basis of $\UU$ and $C_0, c$ depend only on the sub-gaussian norm of the $x_i$.
\end{proposition}

\begin{proof}
Let us fix $\UU \in \EE_p$, with $p=2$ or $3$ and $\Pi_{\UU}$ one of its orthonormal basis. By definition of $\EE_p$, $\dim(\UU) \leq pD$. The rows of $X\Pi_{\UU}$ are orthogonal projections of the rows of $X$ onto $\UU$, so they are still independent sub-gaussian isotropic random vectors. We can therefore apply \citep[Theorem 5.39]{Vers10} on $X\Pi_{\UU} \in \reals^{n \times \dim(\UU)}$. Hence for any $s\geq 0$, with probability at least $1-2\exp(-cs^2)$, the smallest and largest singular values of the rescaled matrix $\frac{X\Pi_{\UU}}{\sqrt{n}}$ written respectively $s_{min}(\frac{X\Pi_{\UU}}{\sqrt{n}})$ and $s_{max}(\frac{X\Pi_{\UU}}{\sqrt{n}})$ are bounded by
\BEQ \label{eq:singval}
1 - C_0\sqrt{\frac{pD}{n}} - \frac{s}{\sqrt{n}} \leq s_{min}\left(\frac{X\Pi_{\UU}}{\sqrt{n}}\right) \leq s_{max}\left(\frac{X\Pi_{\UU}}{\sqrt{n}}\right) \leq 1 + C_0\sqrt{\frac{pD}{n}} +\frac{s}{\sqrt{n}}, 
\EEQ
where $c$ and $C_0$ depend only on the sub-gaussian norm of the $x_i$. Now taking the union bound on all subsets of $\EE_p$, \eqref{eq:singval} holds for any $\UU \in \EE_p$ with probability 
\BEAS
1-2{N \choose p}\exp(-cs^2) & \geq & 1 - 2\left(\frac{eN}{p}\right)^p \exp(-cs^2) \\
& \geq & 1 - 2\exp(1+p\log(N)-cs^2).
\EEAS
Taking $s = \sqrt{\frac{1+p\log(N)}{c}} + \epsilon \sqrt{n}$, we get for all $\UU \in \EE_p$,
\BEQ 
1 - \delta_p - \epsilon \leq s_{min}\left(\frac{X\Pi_{\UU}}{\sqrt{n}}\right) \leq s_{max}\left(\frac{X\Pi_{\UU}}{\sqrt{n}}\right) \leq 1 + \delta_p + \epsilon,  \nonumber
\EEQ
with probability at least $1-\exp(-c\epsilon^2 n)$, where $\delta_p =  C_0\sqrt{\frac{pD}{n}} + \sqrt{\frac{1+p\log(N)}{cn}}$. Therefore  
\BEQ
\frac{1}{\sqrt{n}}\max_{\UU \in \EE_2} \|X\Pi_{\UU}\|_2 \leq 1+\delta_2 +\epsilon. \nonumber
\EEQ
Then \citep[Theorem 5.39]{Vers10} yields
\BEQ
\max_{\UU \in \EE_3} \| I-\frac{1}{n}\Pi_{\UU}^TX^TX \Pi_{\UU} \|_2 \leq 3\max\{\delta_3+\epsilon,(\delta_3+\epsilon)^2\}, \nonumber
\EEQ
hence the desired result.
\end{proof}

Overall here, Proposition \ref{prop:conv} shows that the projected gradient method converges when the contraction factor $\rho$ is strictly less than one. When observations $x_i$ are isotropic independent sub-gaussian, this means
\[
C_0\sqrt{\frac{3D}{n}} < \frac{1}{3} \quad\mbox{and}\quad \sqrt{\frac{1+3\log(N)}{cn}} < \frac{1}{3}
\]
which is also
\BEQ \label{eq:number}
 n =\Omega(D)  \quad\mbox{and}\quad n  = \Omega (\log(N) )
\EEQ
The first condition in~\eqref{eq:number} means that subspaces must be low-dimensional, in our case $D = 3Q$ and we naturally want the number of groups $Q$ to be small. The second condition in~\eqref{eq:number} means that the structure (clustering here) is restrictive enough, \ie~that the number of possible configurations, $N$, is small enough. 

As we show below, in the simple clustering case however, this number of subspaces is quite large, growing essentially as $Q^d$.

\begin{proposition}
The number of subspaces $N$ in $\EE_1$ is lower bounded by
\BEQ
N \geq Q^{d-Q} \nonumber
\EEQ
\end{proposition}
\begin{proof}
$\EE_1$ is indexed by the number of partitions in exactly $Q$ clusters, \ie the Stirling number of second kind $\stirling{d}{Q}$. Standard bounds on the Stirling number of the second kind give
\BEQ \label{eq:stirling_approx}
\frac{1}{2}(Q^2+Q+2)Q^{d-Q-1}-1 \leq \stirling{d}{Q}\leq \frac{1}{2}(ed/Q)^QQ^{d-Q}.
\EEQ
hence $N \geq Q^{d-Q}$.
\end{proof}

This last proposition means that although the intrinsic dimension of our variables is of order $D= 3Q$, the number of subspaces $N$ is such that we need roughly $n \geq 3d\log(Q)$, \ie~approximately as many samples as features, so the clustering structure is  not specific enough to reduce the number of samples required by our algorithm to converge. On the other hand, given this many samples, the algorithm provably converges to a clustered output, which helps interpretation.

As a comparison, classical sparse recovery problems have the same structure \citep{Rao2012}, as $k$-sparse vectors for instance can be described as $\{w : w= Zv,\, Z^T1 =1 \}$ and so are part of a ``union of subspaces''. However in the case of sparse vectors the number of subspaces grows as $d^k$ which means recovery requires much less samples than features.

\subsection{Implementation and complexity}\label{sec:impcomp}
In our implementation we use a backtracking line search on the stepsize $\alpha_t$ that guarantees decreasing of the objective. At each iteration if
\[
\hat{W}_{t+1} = \text{k-means}\left(W_t-\alpha_t(\nabla \Loss(y,X,W_t)+\nabla R(W_t)),Q\right)
\]
decreases the objective value we take $W_{t+1} = \hat{W}_{t+1}$ and we increase the stepsize by a constant factor $\alpha_{t+1} = a\alpha_t$ with $a>1$. If $\hat{W}_{t+1}$ increases the objective value we decrease the stepsize by a constant factor $\alpha_t \leftarrow b\alpha_t$, with $b<1$, output a new $\hat{W}_{t+1}$ and iterate this scheme until $\hat{W}_{t+1}$ decreases the objective value or the stepsize reaches a stopping value $\epsilon$. We observed better results with this line search than with constant stepsize, in particular when the number of samples for clustering features is small.

Complexity of Algorithm~\ref{algo:ICLHCL} is measured by the cost of the projection and the number of iterations until convergence. If approximated by k-means++ the projection step costs $O(t_KQp)$, where $t_K$ is the number of alternating steps, $Q$ is the number of clusters, and $p$ is the product of the dimensions of $V$. When clustering features for regression, the dynamic program of \citet{Zhang03} solving exactly the projection step is in $O(d^2Q)$ and ours for $k$-sparse vectors, detailed in Section~\ref{sec:proj_ksparse_Qcluster}, is in $O(k^2Q)$. We observed convergence of the projected gradient algorithm in less than 100 iterations which makes it highly scalable. As for the convex relaxation the choice of the number of clusters is done given an a priori on the problem.

\section{Sparse and clustered linear models}
Algorithm~\ref{algo:ICLHCL} can also be applied when a sparsity constraint is added to the linear model, provided that the projection is still defined. This scenario arises for instance in text prediction when we want both to select a few relevant words and to group them to reduce dimensionality. Formally the problem of clustering features \eqref{eq:feat} becomes then 
\BEQ\label{eq:least_sparse}
\BA{ll}
\mbox{minimize} & \frac{1}{n}\sum_{i=1}^n \loss\left(y_i,w^T x_i \right) + \frac{\lambda}{2} \|w\|_2^2 \\ 
\mbox{subject to} &  w = SZv, \, S \in \{0,1\}^{d\times k},\, Z \in \{0,1\}^{k \times Q},\, S^T\ones = \ones, \, Z\ones = \ones, \nonumber
\EA\EEQ
in the variables $w \in \reals^d$, $v\in \reals^Q$, $S$ and $Z$, where $S$ is a matrix of $k$ canonical vectors which assigns nonzero coefficients and $Z$ an assignment matrix of $k$ variables in $Q$ clusters.

We develop a new dynamic program to get the projection on $k$-sparse vectors whose non-zero coefficients are clustered in $Q$ groups and apply our previous theoretical analysis to prove convergence of the projected gradient scheme on random instances.

\subsection{Projection on $k$-sparse $Q$-clustered vectors }\label{sec:proj_ksparse_Qcluster}
Let $\mathcal{W}$ be the set of $k$-sparse vectors whose non-zero values can be partitioned in at most $Q$ groups. Given $x \in \reals^d$, we are interested in its projection on $\mathcal{W}$, which we formulate as a partitioning problem. For a feasible $w \in \mathcal{W}$, with $\mathcal{G}_0 = \{ i \; : \; w_i = 0\}$ and $\mathcal{G}_1, \ldots, \mathcal{G}_{Q'}$, with $Q'\leq Q$, the partition of its non-zero values such that $w_i = v_q$ if and only if $i \in \mathcal{G}_q$, the distance between $x$ and $w$ is given by
\[
\|x-w\|_2^2 = \sum_{i \in \mathcal{G}_0}x_i ^2 + \sum_{q = 1}^{Q'} \sum_{i \in \mathcal{G}_q} (x_i-v_q)^2.
\] 
The projection is solution of 
\BEQ \label{eq:k_sparse_Q_clust}
\BA{ll}
\mbox{minimize} & \sum_{i \in \mathcal{G}_0}x_i ^2 + \sum_{q = 1}^{Q'} \sum_{i \in \mathcal{G}_q} (x_i-v_q)^2 \\
\mbox{subject to} & \Card\left(\bigcup_{q=1}^{Q'} \mathcal{G}_q\right) \leq k, \quad 0 \leq Q'\leq Q,
\EA
\EEQ
in the number of groups $Q'$, the partition $\GG =(\GG_0,\ldots,\GG_{Q'})$ of $\{1,\ldots,d\}$ and $v\in \reals^{Q'}$. For a fixed number of non-zero values $k'$, the objective is clearly decreasing in the number of groups $Q'$, which measures the degrees of freedom of the projection, however it cannot exceed $k'$. We will use this argument below to get the best parameter $Q'$.
For a fixed partition  $\GG$, minimization in $v$ gives the barycenters of the $Q'$ groups, $\mu_q = \frac{1}{s_q}\sum_{i \in \mathcal{G}_q} x_i$. Inserting them in \eqref{eq:k_sparse_Q_clust}, the objective can be developed as 
\[
\sum_{i \in \mathcal{G}_0}x_i ^2 + \sum_{q = 1}^{Q'} \sum_{i \in \mathcal{G}_q} x_i^2 + \mu_q^2 - 2\mu_q x_i  = \sum_{i=1}^d x_i^2 - \sum_{q = 1}^{Q'} s_q \mu_q^2.
\]
Splitting this objective between positive and negative barycenters, we get that the minimizer of \eqref{eq:k_sparse_Q_clust} solves  
\BEQ \label{eq:balanced_obj}
\BA{ll}
\mbox{maximize} & \sum_{q \; : \; \mu_q<0}  s_q \mu_q^2 +\sum_{q \; : \; \mu_q>0} s_q \mu_q^2 \\
\mbox{subject to} & \Card\left(\bigcup_{q=1}^{Q'} \mathcal{G}_q\right) \leq k, \quad 0 \leq Q'\leq Q,
\EA
\EEQ
in the number of groups $Q'$ and the partition $\GG =(\GG_0,\ldots,\GG_{Q'})$ of $\{1,\ldots,d\}$, where $\mu_q = \frac{1}{s_q}\sum_{i \in \mathcal{G}_q} x_i$.

We tackle this problem by finding the best balance between the two terms of the objective. We define $f_-(j,q)$ the optimal value of $\sum s_p \mu_p^2$ when picking $j$ points clustered in $q$ groups forming only negative barycenters, \ie~the solution of the problem
\BEQ\label{eq:neg_sum}\tag{$P_-(j,q)$}
\BA{ll}
\mbox{maximize} & \sum_{p = 1}^q s_p \mu_p^2  \\
\mbox{subject to} & \mu_p = \frac{1}{s_p}\sum_{i \in \mathcal{G}_p} x_i <0 \\
& \Card\left(\bigcup_{p=1}^q \mathcal{G}_p\right) = j, \nonumber
\EA
\EEQ
in the partition $\GG = \{\GG_0,\ldots,\GG_q\}$ of $\{1,\ldots,d\}$. We define $f_+(j,q)$ similarly except that it constraints barycenters to be positive. Using remark above on the parameter $Q'$, problem \eqref{eq:balanced_obj} is then equivalent to  
\BEQ\label{eq:balanced_obj_f}
\BA{ll}
\mbox{maximize} & f_-(j,q) + f_+(k'-j,Q'-q) \\
\mbox{subject to} & 0\leq k'\leq k, \; 0\leq j\leq k', \; Q' = \min(k',Q), \; 0 \leq q\leq Q', 
\EA
\EEQ
in variables $j$, $k'$ and $q$.

Now we show that $f_-$ and $f_+$ can be computed by dynamic programming, we begin with $f_-$. Remark that \eqref{eq:neg_sum} is a partitioning problem on the $j$ smallest values of $x$. To see this, let $S_-\subset\{1,\ldots,d\}$ be the optimal subset of indexes taken for \eqref{eq:neg_sum} and $i \in S_-$. If there exists $j\notin S_-$ such that $x_j\leq x_i$, then swapping $j$ and $i$ would increase the magnitude of the barycenter of the group that $i$ belongs to and so the objective. Now for \eqref{eq:neg_sum} a feasible problem, let $\mathcal{G}_1,\ldots,\mathcal{G}_q$ be its optimal partition whose corresponding barycenters are in ascending order and $x_i$ be the smallest value of $x$ in $\mathcal{G}_q$, then necessarily  $\mathcal{G}_1,\ldots,\mathcal{G}_{q-1}$ is optimal to solve $P_-(i-1,q-1)$. We order therefore the values of $x$ in ascending order and use the following dynamic program to compute $f_-$,
\BEQ\label{eq:dyn_prog}
f_-(j,q) = \max_{\substack{q\leq i \leq j \\ \mu (x_i, \ldots , x_j) <0}} f_-(i-1,q-1) + (j-i+1)\mu (x_i,\ldots , x_j )^2,
\EEQ
where $\mu(x_i,\ldots,x_j) = \frac{1}{j-i+1} \sum_{l = i}^j x_l$ can be computed in constant time using that 
\[
\mu(x_i,\ldots,x_j) = \frac{x_i +(j-i)\mu(x_{i+1},\ldots,x_j)  }{j-i+1}.
\]
By convention $f_-(j,q) =-\infty$ if \eqref{eq:dyn_prog} and so \eqref{eq:neg_sum} are not feasible. $f_-$ is initialized as a grid of $k+1$ and $Q+1$ columns such that $f_-(0,q) = 0$ for any $q$, $f_-(j,0) = 0$ and $f_-(j,1) = j\mu(x_1,\ldots,x_j)^2$ for any $j\geq 1$.
Values of $f_-$ are stored to compute \eqref{eq:balanced_obj_f}. Two auxiliary variables $I_-$ and $\mu_-$ store respectively the indexes of the smallest value of $x$ in group $\GG_q$ and the barycenter of the group $\GG_q$, defined by
\BEAS
I_-(j,q) & = & \argmax_{\substack{q\leq i \leq j \\ \mu (x_i, \ldots , x_j) <0}} f_-(i-1,q-1) + (j-i+1)\mu (x_i,\ldots,x_j)^2, \\
\mu_-(j,q) & = & \mu(x_i,\ldots,x_j),\quad i = I_-(j,q).
\EEAS
$I_-$ and $\mu_-$ are initialized by $I_-(j,1) = 1$ and $\mu_-(j,1) = \mu(x_1,\ldots,x_j)$.
The same dynamic program can be used to compute $f_+$, $I_+$ and $\mu_+$, defined similarly as $I_-$ and $\mu_-$, by reversing the order of the values of $x$. A grid search on $f(j,q,k') = f_-(j,q) + f_+(k'-j,Q'-q)$, with $Q'=\min(k',Q)$, gives the optimal balance between positive and negative barycenters. A backtrack on $I_-$ and $I_+$ finally gives the best partition and the projection with the associated barycenters given in $\mu_-$ and $\mu_+$.

Each dynamic program needs only to build the best partitions for the $k$ smallest or largest partitions so their complexity is in $O(k^2Q)$. The complexity of the grid search is $O(k^2Q)$ and the complexity of the backtrack is $O(Q)$. The overall complexity of the projection is therefore $O(k^2Q)$. 

\subsection{Convergence}
Our theoretical convergence analysis can directly be applied to this setting for a problem with squared loss without regularization. The feasible set is again a union of subspaces
\[
w \in \bigcup_{\substack{S \in \{0,1\}^{d\times k},\, S^T\ones = \ones \\ Z \in \{0,1\}^{k \times Q},\, Z\ones = \ones}} \{w : w = SZv\}.
\]
However the number of largest subspaces in terms of inclusion order is smaller. They are defined by selecting $k$ features among $d$ and partitioning these $k$ features into $Q$ groups so that their number is $N = {d \choose k}\stirling{k}{Q}$. Using classical bounds on the binomial coefficient and \eqref{eq:stirling_approx}, we have for $k\geq3$, $Q\geq3$, 
\[
N\leq d^kk^QQ^{k-Q}.
\]
Our analysis thus predicts that only 
\[
n\geq 36\max\left\{QC_0^2,\; \frac{1}{c}(k\log d + Q \log(k) + (k-Q)\log(Q))\right\}
\]
isotropic independent sub-Gaussian samples are sufficient for the projected gradient algorithm to converge. It produces $Q+1$ cluster of features, one being a cluster of zero features, reducing dimensionality, while needing roughly as many samples as non-zero features.

\section{Numerical Experiments}\label{sec:experiments}
We now test our methods, first on artificial datasets to check their robustness to noisy data. We then test our algorithms for feature clustering on real data extracted from movie reviews. While our approach is general and applies to both features and samples, we observe that our algorithms compare favorably with specialized algorithms for these tasks.

\subsection{Synthetic dataset}
\subsubsection{Clustering constraint on sample points\label{sec:synth_samp}}
We test the robustness of our method when the information of the regression problem leading to the partition of the samples lies in a few features.
We generate $n$ data points $(x_i,y_i)$ for $i=1,\ldots,n$, with $x_i \in \mathbb{R}^d$, $d= 8$, and $y_i \in \mathbb{R}$, divided in $Q=3$ clusters corresponding to regression tasks with weight vectors $v_q$. 
Regression labels for points $x_i$ in cluster $\GG_q$ are given by $y_i = v_q^Tx_i + \eta_y$, where $\eta_y \sim \mathcal{N}(0,\sigma_y^2)$. We test the robustness of the algorithms to the addition of noisy dimensions by completing $x_i$ with $d_n$ dimensions of noise $\eta_d \sim \mathcal{N}(0,\sigma_d)$. For testing the models we take the difference between the true label and the best prediction such that the Mean Square Error (MSE) is given by 
\BEQ\label{eq:MSE_samples}
\Loss(y,X,W) = \frac{1}{2n}\sum_{i=1}^n\min_{q=1\ldots,Q} (y_i-v_q^Tx_i)^2.
\EEQ
The results are reported in Table~\ref{table:clustPoints} where the intrinsic dimension is 10 and the proportion of dimensions of noise $d_n/(d+d_n)$ increases. On the algorithmic side,  ``Oracle" refers to the least-squares fit given the true assignments, which can be seen as the best achievable error rate, AM refers to alternate minimization, PG refers to projected gradient with squared loss, CG refers to conditional gradient and RC to regression clustering as proposed by \citet{Zhang03}, implemented using the Harmonic K-means formulation. PG, CG and RC were followed by AM refinement. 1000 points were used for training, 100 for testing. The regularization parameters were 5-fold cross-validated using a logarithmic grid. Noise on labels is $\sigma_y = 10^{-1}$ and noise on added dimensions is $\sigma_d=1$. Results were averaged over 50 experiments with figures after the $\pm$ sign corresponding to one standard deviation.

\begin{table}[h]
\begin{center}
\begin{tabular}{|l|c|c|c|c|c|c|c|}
\hline
&p = 0&p = 0.25&p = 0.5&p = 0.75&p = 0.9&p = 0.95\\\hline
Oracle&0.52$\pm$0.08&0.55$\pm$0.07&0.55$\pm$0.10&0.58$\pm$0.09&0.71$\pm$0.11&1.17$\pm$0.18\\\hline
AM&\textbf{0.52}$\pm$0.08&\textbf{0.55}$\pm$0.07&5.57$\pm$4.11&6.93$\pm$14.39&101.08$\pm$55.49&133.48$\pm$52.20\\\hline
\textbf{PG}&1.53$\pm$7.13&3.98$\pm$17.65&\textbf{3.20}$\pm$13.23&5.64$\pm$20.50&91.33$\pm$39.32&\textbf{131.48}$\pm$50.90\\\hline
\textbf{CG}&0.87$\pm$2.45&1.16$\pm$4.29&3.64$\pm$11.02&\textbf{5.43}$\pm$14.33&91.19$\pm$53.00&136.57$\pm$58.60\\\hline
RC&\textbf{0.52}$\pm$0.08&\textbf{0.55}$\pm$0.07&5.59$\pm$20.27&13.45$\pm$28.76&\textbf{59.19}$\pm$37.97&135.77$\pm$66.96\\\hline
\end{tabular}
\end{center}
\caption{Test MSE given by \eqref{eq:MSE_samples} along proportion of added dimensions of noise $p = d_n/(d+d_n)$. 
\label{table:clustPoints}}
\end{table}

All algorithms perform similarly, RC and AM get better results without added noise. None of the present algorithms get a significantly better behavior with a majority of noisy dimensions.

\subsubsection{Clustering constraint on features}
We test the robustness of our method when with the number of training samples or the level of noise in the labels. We generate $n$ data points $(x_i,y_i)$ for $i=1,\ldots,n$ with $x_i \in \mathbb{R}^d$, $d= 100$, and $y_i \in \mathbb{R}$. Regression weights $w$ have only $5$ different values $v_q$ for $q=1,\ldots,5$, uniformly distributed around 0.
Regression labels are given by $y_i = w^T\x_i + \eta$, where $\eta \sim \mathcal{N}(0,\sigma^2)$. 
We vary the number of samples $n$ or the level of noise $\sigma$ and measure $\|w_*-\hat{w}\|_2$, the $l_2$ norm of the difference between the true vector of weights $w^*$ and the estimated ones $\hat{w}$.

In Table~\ref{table:clustFeatures} and ~\ref{table:clustFeaturesSig}, we compare the proposed algorithms to Least Squares (LS), Least Squares followed by K-means on the weights (using associated centroids as predictors) (LSK) and OSCAR \citep{bondell08}. For OSCAR we used a submodular approach \citep{Bach2012} to compute the corresponding proximal algorithm, which makes it scalable.  ``Oracle" refers to the Least Square solution given the true assignments of features and can be seen as the best achievable error rate. Here too, PG refers to projected gradient with squared loss (initialized with the solution of Least Square followed by k-means), CG refers to conditional gradient, CGPG refers to conditional gradient followed by PG. When varying the number of samples, noise on labels is set to $\sigma=0.5$ and when varying level of noise $\sigma$ number of samples is set to $n = 150$. Parameters of the algorithms were all cross-validated using a logarithmic grid. Results were averaged over 50 experiments and figures after the $\pm$ sign correspond to one standard deviation.

\begin{table}[h]
\begin{center}
\begin{tabular}{|l|c|c|c|c|c|}
\hline
&$n$ = 50&$n$ = 75&$n$ = 100&$n$ = 125&$n$ = 150\\\hline
Oracle&0.16$\pm$0.06&0.14$\pm$0.04&0.10$\pm$0.04&0.10$\pm$0.04&0.09$\pm$0.03\\\hline
LS&61.94$\pm$17.63&51.94$\pm$16.01&21.41$\pm$9.40&1.02$\pm$0.18&0.70$\pm$0.09\\\hline
LSK&62.93$\pm$18.05&57.78$\pm$17.03&10.18$\pm$14.96&0.31$\pm$0.19&0.19$\pm$0.12\\\hline
\textbf{PG}&63.31$\pm$18.24&52.72$\pm$16.51&5.52$\pm$14.33&\textbf{0.14}$\pm$0.09&\textbf{0.09}$\pm$0.04\\\hline
\textbf{CG}&61.81$\pm$17.78&52.59$\pm$16.58&17.24$\pm$13.87&1.20$\pm$1.38&1.05$\pm$1.37\\\hline
\textbf{CGPG}&62.29$\pm$18.15&\textbf{50.15}$\pm$17.43&\textbf{0.64}$\pm$2.03&0.15$\pm$0.19&0.17$\pm$0.53\\\hline
OS&\textbf{61.54}$\pm$17.59&52.87$\pm$15.90&11.32$\pm$7.03&1.25$\pm$0.28&0.71$\pm$0.10\\\hline
\end{tabular}
\end{center}
\caption{Measure of $\|w_*-\hat{w}\|_2$, the $l_2$ norm of the difference between the true vector of weights $w^*$ and the estimated ones $\hat{w}$ along number of samples $n$.
\label{table:clustFeatures}}
\end{table}

\begin{table}[h]
\begin{center}
\begin{tabular}{|l|c|c|c|c|}
\hline
&$\sigma$ = 0.05&$\sigma$ = 0.1&$\sigma$ = 0.5&$\sigma$ = 1\\\hline
Oracle&0.86$\pm$0.27&1.72$\pm$0.54&8.62$\pm$2.70&17.19$\pm$5.43\\\hline
LS&7.04$\pm$0.92&14.05$\pm$1.82&70.39$\pm$9.20&140.41$\pm$18.20\\\hline
LSK&1.44$\pm$0.46&2.88$\pm$0.91&19.10$\pm$12.13&48.09$\pm$27.46\\\hline
\textbf{PG}&\textbf{0.87}$\pm$0.27&\textbf{1.74}$\pm$0.52&\textbf{9.11}$\pm$4.00&26.23$\pm$18.00\\\hline
\textbf{CG}&23.91$\pm$36.51&122.31$\pm$145.77&105.45$\pm$136.79&155.98$\pm$177.69\\\hline
\textbf{CGPG}&1.52$\pm$3.13&140.83$\pm$710.32&17.34$\pm$53.31&\textbf{24.80}$\pm$16.32\\\hline
OS&14.43$\pm$2.45&18.89$\pm$3.46&71.00$\pm$10.12&140.33$\pm$18.83\\\hline
\end{tabular}
\end{center}
\caption{Measure of $\|w_*-\hat{w}\|_2$, the $l_2$ norm of the difference between the true vector of weights $w^*$ and the estimated ones $\hat{w}$ along level of noise $\sigma$.
\label{table:clustFeaturesSig}}
\end{table}

We observe that both PG and CGPG give significantly better results than other methods and even reach the performance of the Oracle for $n>d$ and for small $\sigma$, while for $n\leq d$ results are in the same range. 

\subsection{Real data}
\subsubsection{Predicting ratings from reviews using groups of words.}

We perform ``sentiment" analysis of newspaper movie reviews. We use the publicly available dataset introduced by~\citet{pang05} which contains movie reviews paired with star ratings.  We treat it as a regression problem, taking responses for $y$ in $(0,1)$ and word frequencies as covariates. The corpus contains $n = 5006$ documents and we reduced the initial vocabulary to $d = 5623$ words by eliminating stop words, rare words and words with small TF-IDF mean on whole corpus. We evaluate our algorithms for regression with clustered features against standard regression approaches: Least-Squares (LS),  and Least-Squares followed by k-means on predictors (LSK), Lasso and Iterative Hard Thresholding (IHT). We also tested our projected gradient with sparsity constraint, initialized by the solution of LSK (PGS) or by the solution of CG (CGPGS). Number of clusters, sparsity constraints and regularization parameters were 5-fold cross-validated using respectively grids going from 5 to 15, $d/2$ to $d/5$ and logarithmic grids. Cross validation and training were made on 80\% on the dataset and tested on the remaining 20\% it gave $Q=15$ number of clusters and $d/2$ sparsity constraint for our algorithms. Results are reported in Table~\ref{table:resultsReview}, figures after the $\pm$ sign correspond to one standard deviation when varying the training and test sets on 20 experiments.  

All methods perform similarly except IHT and Lasso whose hypotheses does not seem appropriate for the problem. Our approaches have the benefit to reduce dimensionality from 5623 to 15 and provide meaningful cluster of words.
The clusters with highest absolute weights are also the ones with smallest number of words, which confirms the intuition that only a few words are very discriminative. We illustrate this in Table~\ref{table:wordsReview}, picking randomly words of the four clusters within which associated predictor weights $v_q$ have largest magnitude.

\begin{table}[h]
\begin{center}
\begin{tabular}{|c|c|c|c|c|c|}
\hline
LS&LSK&PG&CG&CGPG&OS\\\hline
1.51$\pm$0.06&1.53$\pm$0.06&1.52$\pm$0.06&1.58$\pm$0.07&1.49$\pm$0.08&1.47$\pm$0.07\\\hline
\end{tabular}

\begin{tabular}{|c|c|c|c|}
\hline
PGS&CGPGS&IHT&Lasso \\\hline
1.53$\pm$0.06&1.49$\pm$0.07&2.19$\pm$0.12&3.77$\pm$0.17\\\hline
\end{tabular}
\caption{100 $\times$ mean square errors for predicting movie ratings associated with reviews. 
\label{table:resultsReview}}
\end{center}
\end{table}

\begin{table}[h]
\begin{center}
\scalebox{0.88}{
\begin{tabular}{| l | c |}
\hline \textbf{First and Second Cluster} & bad, awful, \\
  (negative) & worst, boring, ridiculous,\\
  sizes 1 and 7   & watchable, suppose, disgusting, \\
\hline \textbf{Last and Before Last Cluster  } & perfect,hilarious,fascinating,great \\
 (positive) & wonderfully,perfectly,goodspirited,\\
sizes 4 and 40 &     world, intelligent,wonderfully,unexpected,gem,recommendation, \\
& excellent,rare,unique,marvelous,good-spirited,\\
& mature,send,delightful,funniest\\ \hline
\end{tabular}
}
\end{center}
\caption{Clustering of words on movie reviews. 
We show clusters of words within which associated predictor weights $v_q$ have largest magnitude. First and second one are associated to a negative coefficient and therefore bad feelings about movies, last and before last ones to a positive coefficient and good feelings about movies.
\label{table:wordsReview}}
\end{table}

\subsection*{Acknowledgements}
AA is at CNRS, at the D\'epartement d'Informatique at \'Ecole Normale Sup\'erieure, 2 rue Simone Iff, 75012 Paris, France. FB is at the D\'epartement d'Informatique at \'Ecole Normale Sup\'erieure and INRIA, Sierra project-team, PSL Research University. The authors would like to acknowledge support from a starting grant from the European Research Council (ERC project SIPA), an AMX fellowship, the MSR-Inria Joint Centre, as well as support from the chaire {\em \'Economie des nouvelles donn\'ees}, the {\em data science} joint research initiative with the {\em fonds AXA pour la recherche} and a gift from Soci\'et\'e G\'en\'erale Cross Asset Quantitative Research.

{
\bibliographystyle{plainnat}
\bibliography{main}
}

\clearpage
\section{Appendix}
\subsection{Formulations for classification \label{sec:classif}}
We present here formulations of clustering  either features or samples when our task is to classify samples into $K$ classes. For both settings we assume that $n$ sample points are given, represented by the matrix $X = (x_1,...,x_n)^T \in \reals^{n \times d}$ and corresponding labels $Y = (y_1,\ldots,y_K)\in \{0,1\}^{n\times K}$.
\subsubsection{Clustering features for classification}
Here we search $K$ predictors $W = (w_1,\ldots,w_K)$, each of them having features clustered in $Q$ groups $\{\GG_1,\ldots,\GG_Q\}$ such that for any $k$, $w_k^{j} = v_k^{q}$ if feature $j$ is in group $q$. Partition of the features is shared by all predictors but each has different centroids represented in the vector $v_k$. Using an assignment matrix $Z$ and the matrix of centroids $V=(v_1,\ldots,v_k)$, our problem can therefore be written 

\BEQ\label{eq:feat_class}
\BA{ll}
\mbox{minimize} & \frac{1}{n}\sum_{i=1}^n \loss\left(y_i, x_i^TW \right) + \frac{\lambda}{2} \|W\|_F^2 \\
\mbox{subject to} & W = ZV, \, Z \in \{0,1\}^{d\times Q},\, Z\ones = \ones \nonumber
\EA
\EEQ
in variables $W\in \reals^{d\times K}$, $V\in \reals^{Q\times K}$ and $Z$. $\loss\left(y_i, x_i^TW \right)$ is a squared or logistic multiclass loss and regularization can either be seen as a standard $\ell_2$ regularization on the $w_k$ or a weighted regularization on the centroids $v_k$.

\subsubsection{Clustering samples for classification}
Here our objective is to form $Q$ groups $\{\GG_1,\ldots,\GG_Q\}$ of sample points to maximize the within-group prediction performance. For classification, within each group $\GG_q$, samples are predicted using a common matrix of predictors $V^{q} = (v_1^{q},\ldots,v_K^q)$. Our problem can be written

\BEQ\label{eq:cl-smp-class}
\mbox{minimize}~ \frac{1}{n}\sum_{i \in \mathcal{G}_q} \loss\left(y_i, x_i^T V^q \right) + \frac{\lambda}{2} \sum_{q=1}^Q s_q \|V^q\|^2_F 
\EEQ
in the variables $V = (V^1,\ldots,V^Q) \in \reals^{d \times K \times Q}$ and $\GG = (\GG_1,\ldots,\GG_Q)$ such that $\GG$ is a partition of the $n$ samples.  $\loss\left(y_i, x_i^TW \right)$ is a squared or logistic multiclass and $\frac{\lambda}{2} \sum_{q=1}^Q s_q \|V^q\|^2_F$ is a weighted regularization. Using an assignment matrix $Z \in \{0,1\}^{n\times Q}$ and auxiliary variables $(W^1,\ldots,W^n) \in \reals^{d \times K\times n}$ such that $W^i=V^q$ if $i \in \mathcal{G}_q$, problem~\eqref{eq:cl-smp-class} can be rewritten
\BEQ
\BA{ll}
\mbox{minimize} & \frac{1}{n}\sum_{i = 1}^n \loss\left(y_i,{W^i}^T x_i \right) + \frac{\lambda}{2} \sum_{i=1}^n  \|W^i\|^2_F \\
\mbox{subject to} & \tilde{W}^T = Z\tilde{V}^T, \, Z \in \{0,1\}^{n\times Q}, \, Z\ones = \ones, \nonumber \\
\EA
\EEQ
in the variables $W \in \reals^{d\times K\times n}$, $V \in \reals^{d\times K\times Q}$ and $Z$, where $\tilde{W} = (\Vect(W^1),\ldots,\Vect(W^n)), \, \tilde{V} = (\Vect(V^1),\ldots,\Vect(V^Q))$ and for a matrix $A$, $\Vect(A)$ concatenates its columns into one vector.

Remark that in that case we must have $K>Q$ otherwise we output more possible answers than classes (in that case the problem is ill-posed).
\subsection{Clustered multitask \label{sec:multitask}}  
Our framework applies also to transfer learning by clustering similar tasks. Given a set of $K$ supervised tasks like regression or binary classification, transfer learning aims at jointly solving these tasks, hoping that each task can benefit from the information given by other tasks.
For simplicity, we illustrate the case of multi-category classification, which can be extended to the general multitask setting.
When performing classification with one-versus-all majority vote, we train one binary classifier for each class vs. all others. Using a regularizing penalty such as the squared  $\ell_2$ norm, the problem of multitask learning can be cast as
\BEQ
\BA{ll}
\mbox{minimize}~ & \frac{1}{n} \sum_{k=1}^K \sum_{i = 1}^n \loss(y_i^k,w_k^T\x_i) +\lambda \sum_{i=1}^K \|w_k\|_2^2. 
\EA
\EEQ
in the matrix variable $W = (w_1,\ldots,w_k)\in\reals^{d \times K}$ of classifier vectors (one column per task). We write $\Loss(y,X,W)$ and $R(W)$ the first and second term of this problem. Various strategies are used to leverage the information coming from related tasks, such as low rank \cite{argyriou08} or structured norm penalties \cite{ciliberto2015convex} on the matrix of classifiers $W$. Here we follow the clustered multitask setting introduced in \cite{Jacob09}. Namely we add a penalty $\Omega$ on the classifiers $(w_1,\ldots,w_K)$ which enforce them to be clustered in $Q$ groups $\GG_1,\ldots,\GG_Q$ around centroids $V = (v_1,\ldots,v_Q) \in \reals^{d\times Q}$. This penalty can be decomposed in

\begin{itemize}[noitemsep,leftmargin =*]
\item A measure of the {\bf norm of the barycenter of centers} $\bar{v} = \frac{1}{K} \sum_{q=1}^Q s_qv_q$
\[
\Omega_{mean}(V)  =  \frac{\lambda_m }{2}K||\bar{v}||_2^2  
\]
\item A measure of the {\bf variance between clusters}
\[ 
\Omega_{between}(V) = \frac{\lambda_b}{2}\sum_{q=1}^Q s_q||v_q - \bar{v}||_2^2  
\]
\item A measure of the {\bf variance within clusters}
\[
\Omega_{within}(W,V) =  \frac{\lambda_w}{2}\sum_{q=1}^Q\sum_{i \in  \GG_q }||w_i - v_q||^2_2 
\]
\end{itemize}
The total penalty $\Omega(W,V) = \Omega_{mean}(V) + \Omega_{between}(V) +\Omega_{within}(W,V)$ is illustrated in Figure~\ref{fig:Clustered_Penalty}. 

\begin{figure}[H]
\begin{center}
\includegraphics[scale=0.5]{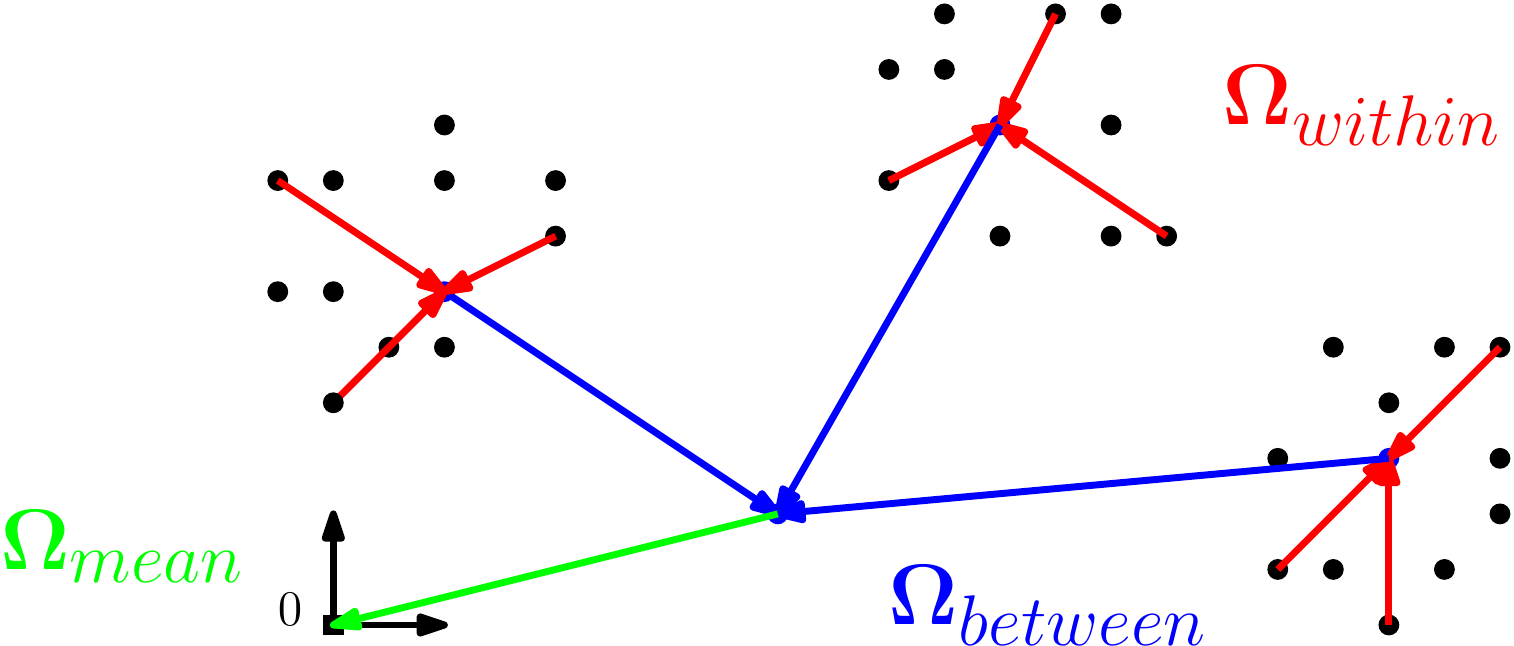}
\end{center}
\caption{Decomposed clustering penalty on $K$ classes in the space of classifier vectors.
\label{fig:Clustered_Penalty}} 
\end{figure}

The clustered multitask learning problem can then be written using an assignment matrix $Z$ and an auxiliary variable $W$ Denoting $\Pi = \idm - \frac{\ones\ones^T}{K}$ the centering matrix of the $K$ classes, we develop each term of the penalty,

\BEAS
\Omega_{mean}(V,Z) & = & \frac{\lambda_M}{2} \Tr(VZ^T(\idm-\Pi)ZV^T), \\ 
\Omega_{between}(V,Z) & = &  \frac{\lambda_B}{2} \Tr(VZ^T\Pi ZV^T), \\
\Omega_{within}(W,V,Z) & = & \frac{\lambda_W}{2}||W-VZ^T||^2_F. 
\EEAS
Using $\tilde{W} = VZ^T$ the total penalty can then be written 
\[ \Omega(W,\tilde{W}) =  \frac{\lambda_M}{2} \Tr(\tilde{W}(\idm-\Pi)\tilde{W}^T) + \frac{\lambda_B}{2} \Tr(\tilde{W}\Pi \tilde{W}^T) + \frac{\lambda_W}{2}||W-\tilde{W}||^2_F, \]
and the problem is 
\begin{align*}
\mbox{minimize}~ & \Loss(y,X,W) + R(W) + \Omega(W,\tilde{W}) \\
\mbox{s.t.}~ & \tilde{W}^T = ZV^T, \quad Z \in \{0,1\}^{K\times Q},\quad Z\ones = \ones,
\end{align*}
in variables $W \in \reals^{d \times K}$, $\tilde{W} \in \reals^{d\times K}$, $V\in \reals^{d\times Q}$ and $Z$.

\subsection{Convex relaxations formulations}\label{sec:cvx_relax_other}

\subsubsection{Clustering samples for regression task}
We use a squared loss $l(\hat{y},y) = \frac{1}{2}(y-\hat{y})^2$ in \eqref{eq:cl-smp} and minimize in $V$ to get a clustering problem that we can tackle using Frank-Wolfe method. We fix a partition $\GG$ and define for each group $\GG_q = \{k_1,\ldots,k_{s_q}\}\subset\{1,\ldots,d\}$, the matrix $E\in\{0,1\}^{s_q \times n}$ that picks the $s_q$ points of $\GG_q$, \ie~$(E_q)_{ij} = 1$ if $j = k_i$ and $0$ otherwise.
Therefore $y_q = E_qy\in\reals^{s_q}$ and $X_q=E_qX\in\reals^{s_q \times d}$ are respectively the vector of labels and the matrix of sample vectors of the group $\GG_q$. We naturally have $E_qE_q^T = \idm$ as rows of $E_q$ are orthonormal and $E_q^TE_q$ is a diagonal matrix where $Z_q = \diag(E_q^TE_q)\in \{0,1\}^n$ is the assignment vector in group $\GG_q$, \ie~$(Z_q)_j=1$ if $j\in\GG_q$ and $0$ otherwise. $Z = (Z_1,\ldots,Z_Q)$ is therefore an assignment matrix for the partition $\GG$.
 
Minimizing in $v$ and using the Sherman-Woodbury-Morrison formula, we obtain a function of the partition
\BEAS
\tilde{\psi}(\GG) & = & \min_{v_1,...,v_Q}\frac{1}{2n}\sum_{q=1}^Q\|y_q-X_qv_q\|_2^2 + \frac{\lambda}{2}\sum_{q=1}^Qs_q\|v_q\|_2^2 \\
& = & \frac{1}{2n}\sum_{q=1}^Q\|y_q\|_2^2 - y_q^TX_q(s_q\lambda n  \idm+X_q^TX_q)^{-1}X_q^Ty_q \\
& = & \frac{1}{2n} \sum_{q=1}^Qy_q^T(\idm+\frac{1}{s_q \lambda n}X_qX_q^T)^{-1}y_q. \\
\EEAS
Formulating terms of the sum as solutions of an optimization problem, we get 
\BEAS
\tilde{\psi}(\GG) & = & \frac{1}{2n} \sum_{q=1}^Q\max_{\alpha_q\in \reals^{s_q}} -\alpha_q^T(\idm + \frac{1}{s_q \lambda n}X_qX_q^T)\alpha_q + 2y_q^T\alpha_q \\
& = & \frac{1}{2n} \max_{\substack{\alpha = (\alpha_1;\ldots;\alpha_Q) \\ \alpha_q \in \reals^{s_q}}} \sum_{q=1}^Q-\alpha_q^T(\idm + \frac{1}{s_q \lambda n}X_qX_q^T)\alpha_q + 2y_q^T\alpha_q,
\EEAS
where $(\alpha_1;\ldots;\alpha_Q) = (\alpha_1^T,\ldots,\alpha_Q^T)^T$ stacks vectors $\alpha_q$ in one vector of size $\sum_{q=1}^Qs_q=n$.
Using that $E=(E_1;\ldots;E_Q) = (E_1^T,\ldots,E_Q^T)^T \in \{0,1\}^{n\times n}$ is an orthonormal matrix, we make the change of variable $\beta = E^T\alpha$ (and so $\alpha = E\beta$) such that for  $\alpha = (\alpha_1;\ldots;\alpha_Q)$, $\alpha_q \in \reals^{s_q}$, $\alpha_q = E_q\beta$. Decomposing $X_q$ and $y_q$ and using $E_q^TE_q = \diag(Z_q)$, we get 
\BEAS
\tilde{\psi}(\GG)& = & \frac{1}{2n} \max_{\beta \in \reals^n} \sum_{q=1}^Q-\beta^TE_q^T(\idm + \frac{1}{s_q \lambda n}X_qX_q^T)E_q\beta + 2y_q^TE_q\beta \\
& = & \frac{1}{2n} \max_{\beta \in \reals^n} \sum_{q=1}^Q-\beta^TE_q^T(\idm + \frac{1}{s_q \lambda n}E_qXX^TE_q^T)E_q\beta + 2y^TE_q^TE_q\beta \\
& = & \frac{1}{2n} \max_{\beta \in \reals^n} \sum_{q=1}^Q-\beta^T\diag(Z_q)\beta - \frac{1}{s_q \lambda n}\beta^T\diag(Z_q)XX^T\diag(Z_q)\beta + 2y^T\diag(Z_q)\beta.
\EEAS
For $q$ fixed, $\left(\frac{1}{s_q}\diag(Z_q)XX^T\diag(Z_q)\right)_{ij} = \frac{1}{s_q}x_i^Tx_j$ if $(i,j)\in \GG_q$ and $0$ otherwise. So 
\[
\sum_{q=1}^Q\frac{1}{s_q}\diag(Z_q)XX^T\diag(Z_q) = XX^T\circ M,
\]
where $M = Z(Z^TZ)^{-1}Z^T$ is the normalized equivalence matrix of the partition $\GG$ and $\circ$ denotes the Hadamard product. Using $\sum_{q=1}^Q \diag(Z_q) = \idm$, we finally get a function of the equivalence matrix
\BEAS
\psi(M) & = & \frac{1}{2n}\max_{\beta\in \reals^n} -\beta^T(\idm+\frac{1}{\lambda n}XX^T\circ M)\beta + 2y^T\beta \\
& = & \frac{1}{2n} y^T(\idm +\frac{1}{\lambda n}XX^T\circ M)^{-1}y.
\EEAS
Its gradient is given by
\[
 \nabla \psi(M) =  - \frac{1}{2\lambda n^2 } XX^T \circ \left( (\idm + \frac{1}{\lambda n } XX^T \circ M)^{-1}yy^T (\idm + \frac{1}{\lambda n } XX^T \circ M)^{-1} \right).
\]
Algorithm~\ref{algo:condGrad} can be applied to minimize $\psi$. The linear oracle can indeed be computed with k-means using that the gradient is negative semi-definite.  
For a fixed $Z$, the linear predictors $v_q$ for each cluster of points are given by
\BEAS
v_q & = & (n \lambda s_q \idm+ X_q^TX_q)^{-1} X_q^T y_q \\
& = & (n \lambda s_q \idm+ X^TE_q^TE_qX)^{-1} X^TE_q^TE_q y \\
& = & (n \lambda s_q \idm+ X^T\diag(Z_q)X)^{-1} X^T\diag(Z_q) y.
\EEAS

\subsubsection{Convex relaxations for classification}
We observe that convex relaxations for classification derive from computations of the convex relaxations for regression by replacing vector of labels $y$ by the corresponding matrix of labels $Y$.

\end{document}